\documentclass{article}
\usepackage[margin=1.25in]{geometry}
\usepackage{hyperref}       
\usepackage{url}            
\usepackage{booktabs}       
\usepackage{amsfonts}       
\usepackage{nicefrac}       
\usepackage{microtype}      
\usepackage{algorithm}
\usepackage{algorithmic}
\usepackage{multirow}
\usepackage{bm}
\usepackage{amsmath,amssymb,amsfonts}
\usepackage{epsfig}

\newtheorem{lemma}{Lemma}
\newtheorem{proof}{Proof}

\newcommand{\x}{\mathbf{x}}
\newcommand{\y}{\mathbf{y}}
\newcommand{\z}{\mathbf{z}}
\newcommand{\f}{\mathbf{f}}

\newcommand{\W}{\bm{\mathbf{W}}}
\newcommand{\U}{\bm{\mathbf{U}}}

\newcommand{\A}{\bm{\mathbf{A}}}

\newcommand{\<}{\left\langle}
\renewcommand{\>}{\right\rangle}
\newcommand{\prox}{\mbox{prox}}

\newcommand{\argmin}{\mbox{argmin}}

\begin{document}

\title{\textbf{Optimization Algorithm Inspired Deep Neural Network Structure Design}}

\author{Huan Li\footnote{Peking University. Email: lihuanss@pdu.edu.cn} \and
Yibo Yang\footnote{Peking University. Email: ibo@pku.edu.cn. H. Li and Y. Yang have equal contributions.} \and
Dongmin Chen\footnote{Peking University. Email: dongminchen@pku.edu.cn} \and
Zhouchen Lin\footnote{Peking University. Email: zlin@pku.edu.cn. Z. Lin is the corresponding author.}
}

\maketitle

\begin{abstract}
Deep neural networks have been one of the dominant machine learning approaches in recent years. Several new network structures are proposed and have better performance than the traditional feedforward neural network structure. Representative ones include the skip connection structure in ResNet and the dense connection structure in DenseNet. However, it still lacks a unified guidance for the neural network structure design. In this paper, we propose the hypothesis that the neural network structure design can be inspired by optimization algorithms and a faster optimization algorithm may lead to a better neural network structure. Specifically, we prove that the propagation in the feedforward neural network with the same linear transformation in different layers is equivalent to minimizing some function using the gradient descent algorithm. Based on this observation, we replace the gradient descent algorithm with the
heavy ball algorithm and Nesterov's accelerated gradient descent algorithm, which are faster and inspire us to design new and better network structures. ResNet and DenseNet can be considered as two special cases of our framework. Numerical experiments on CIFAR-10, CIFAR-100 and ImageNet verify the advantage of our optimization algorithm inspired structures over ResNet and DenseNet.
\end{abstract}

\section{Introduction}
Deep neural networks have become a powerful tool in machine learning and have achieved remarkable success in many computer vision and image processing tasks, including classification \cite{hinton2012}, semantic segmentation \cite{long2015} and object detection \cite{Girshick2015,Ren-2016}. After the breakthrough result in the ImageNet classification challenge \cite{hinton2012}, different kinds of neural network architectures have been proposed and the performance is improved year by year. GoogLeNet \cite{Szegedy2015} seems to achieve the bottleneck of performance in 2014 with the traditional feedforward neural network structure, where the units are connected only with the ones in the following layer. After that, new neural network structures have been proposed. Examples include ResNet \cite{he2015resbet}, DenseNet \cite{huang2016densely} and CliqueNet \cite{Yang_2018_CVPR}, where skip connections and dense connections are adopted to ease the network training and further push the state of the arts.


Although the new structures lead to a significant improvement compared with the traditional feedforward structure, it seems to require profound understandings of practical neural networks and substantial trials of experiments to design effective neural network structures. Thus we believe that the design of neural network structure needs a unified guidance. This paper serves as a preliminary trial towards this goal.
\subsection{Related Work}
There has been extensive work on the neural network structure design. Generic algorithm \cite{schaffer1992,lam2003} based approaches were proposed to find both architectures and weights in the early stage of neural network design. However, networks designed with the generic algorithm perform worse than the hand-crafted ones \cite{verbancsics2013}. \cite{saxena2016} proposed a ``Fabric'' to sidestep the CNN model architecture selection problem and it performs close to the hand-crafted networks. \cite{domhan2015} used Bayesian optimization for network architecture selection and \cite{bergstra2013} used a meta-modeling approach to choose the type of layers and hyper parameters. \cite{kwok1997}, \cite{ma2003} and \cite{cortes2017} used the adaptive strategy that grows the network structure layer by layer from a small network based on some principles, e.g., \cite{cortes2017} minimized some loss value to balance the model complexity and the empirical risk minimization. \cite{baker2016} and \cite{zoph} used the reinforcement learning to search the neural network architecture. All of these are basically heuristic search based approaches. They are difficult to produce effective neural networks if the computing power is insufficient or the search strategy is inefficient as the search space is huge. DNN structure designed via minimizing some loss values is only best for given data. It may not generalize to other datasets if no \emph{regular} structure exists in the network. Although some recently proposed methods that utilize a recurrent neural network and reinforcement learning scheme also achieve impressive results \cite{zoph,shangtang}, they differ from us due to the lack of explicit guidance to indicate where the connections should appear.
\subsection{Motivation}
In this paper, we design the neural network structures based on the inspiration from optimization algorithms. Our idea is motivated by the recent work in the compressive sensing community. Traditional methods for compressive sensing solve a well-defined problem $\min_{\x} \|\A\x-\y\|^2+\alpha\|\x\|_1$\footnote{We denote $\|\x\|=\sqrt{\sum_i\x_i^2}$ and $\|\x\|_1=\sum_i|\x_i|$.} and employ iterative algorithms to solve it, e.g., the ISTA algorithm \cite{Beck-2009-APG} with iterations $\x_{k+1}=\prox_{\frac{\alpha}{L}\|\cdot\|_1}(\x_k-\frac{1}{L}\A^T(\A\x_k-\y))$, where $\prox_{\alpha\|\cdot\|_1}(\x)=\argmin_{\z} \frac{1}{2}\|\z-\x\|^2+a\|\z\|_1$. The iterative algorithms often need many iterations to converge and thus suffer from high computational complexity. \cite{Gregor2010}, \cite{xinbo2016}, \cite{kulkarni2016}, \cite{zhang2017} and \cite{yang-2016-nips} developed a series of neural network based methods for compressive sensing. Their main idea is to train a non-linear feedforward neural network with a fixed depth. At each layer, a linear transformation is applied to the input $\x_k$ and then a nonlinear transformation follows, which can be described as $\x_{k+1}=\Phi(\W_k\x_k)$. In the traditional optimization based compressive sensing, the linear transformation $\A$ is fixed. As a comparison, in the neural network based compressive sensing, $\W_k$ is learnable so that each layer has a different linear transformation matrix. The neural network based compressive sensing often needs much less computation compared with the optimization based ones.

Since ISTA is almost the most popular algorithm for compressive sensing, most of the existing neural network based methods \cite{Gregor2010,xinbo2016,kulkarni2016} are inspired by ISTA and thus have the feedforward structure. \cite{zhang2017} proposed a FISTA-net \cite{Beck-2009-APG} by adding a skip connection to the feedforward structure. However, all these networks are for \emph{image reconstruction}, based on the compressive sensing model. The design methodology of deep neural networks for \emph{image recognition} tasks is still lacking.
\subsection{Contributions}\label{section_constritution}
In this paper, we study the design of the neural network structures for \emph{image recognition} tasks\footnote{We only focus on the part before  SoftMax as SoftMax will be connected to all networks in order to produce label information.}. To make our network structure easy to generalize to other datasets, our methodology separates the structure design and weights search, i.e., we do not consider the optimal weights in the structure design stage. The optimal weights will be searched via training after the structure design. Our methodology is inspired by optimization algorithms. Specifically, our contributions include:
\begin{enumerate}
\item For the standard feedforward neural network that shares the same linear transformation and nonlinear activation function at different layers, we prove that the propagation in the neural network is equivalent to using the gradient descent algorithm to minimize some function $F(\x)$. As a comparison, the neural network based compressed sensing only studied the soft thresholding as the nonlinear activation function and the goal of designing the network is to solve the compressive sensing problem as accurately as possible.
\item Based on the above observation, we propose the hypothesis that a faster optimization algorithm may inspire a better neural network structure. Especially, we give the neural network structures inspired by the heavy ball algorithm and Nesterov's accelerated gradient algorithm, which include ResNet and DenseNet as two special cases.
\item Numerical experiments on CIFAR-10, CIFAR-100 and ImageNet verify that the optimization algorithm inspired neural network structures outperform ResNet and DenseNet. These show that our methodology is very promising.
\end{enumerate}
Our methodology is still preliminary. Although we have shown in some degree the connection between faster optimization algorithms and better deep neural networks, currently we haven't revealed the connection between optimization algorithm and DNN structure design in a theoretically rigorous way. It is an analogy. However, analogy does \emph{not} mean unsolid. For example, DNN is inspired by brain. It is also an analogy and has no strict connections to brain either. However, no one can say that DNN is insignificant or ineffective. 
\section{Reviews of Some Optimization Algorithms}
In this section, we review the gradient descent (GD) algorithm \cite{Bertsekas-book}, the heavy ball (HB) algorithm \cite{polyak1964}, Nesterov's accelerated gradient descent (AGD) algorithm \cite{Nesterov1983} and the Alternating Direction Method of Multipliers (ADMM) \cite{Gabay-1983-DR,lin-2011} to solve the general optimization problem $\min_{\z} f(\z)$.

The gradient descent algorithm is one of the most popular algorithms in practice. It consists of the following iteration\footnote{For direct use in our network design, we fix the stepsize to 1. It can be obtained by scaling the objective function $f(\z)$ such that the Lipschitz constant of $\nabla f(\z)$ is 1.}:
\begin{eqnarray}\label{GD}
\z_{k+1}=\z_k- \nabla f(\z_k).
\end{eqnarray}
The heavy ball algorithm is a variant of the gradient descent algorithm, where a momentum is added after the gradient descent step:
\begin{eqnarray}\label{HB}
\z_{k+1}=\z_k- \nabla f(\z_k)+\beta(\z_k-\z_{k-1}).
\end{eqnarray}
Nesterov's accelerated gradient algorithm has the similar idea with the heavy ball algorithm, but it uses the momentum in another way:
\begin{eqnarray}
\begin{aligned}\label{AGD1}
&\y_k=\z_k+\frac{\theta_k(1-\theta_{k-1})}{\theta_{k-1}}(\z_k-\z_{k-1}),\\
&\z_{k+1}=\y_k- \nabla f(\y_k).
\end{aligned}
\end{eqnarray}
where $\theta_{k}$ is computed via $\frac{1-\theta_{k}}{\theta_{k}^2}=\frac{1}{\theta_{k-1}^2}$ and $\theta_0=1$\footnote{When $f(\z)$ is $\mu$-strongly convex and its gradient is $L$-Lipschitz continuous, $\frac{\theta_k(1-\theta_{k-1})}{\theta_{k-1}}$ is fixed at $\frac{\sqrt{L}-\sqrt{\mu}}{\sqrt{L}+\sqrt{\mu}}$.}. When $f(\z)$ is $\mu$-strongly convex\footnote{I.e., $f(\y)\geq f(\x)+\<\nabla f(\x),\y-\x\>+\frac{\mu}{2}\|\y-\x\|^2$.} and its gradient is $L$-Lipschitz continuous\footnote{I.e., $\|\nabla f(\y)-\nabla f(\x)\|\leq L\|\y-\x\|$.}, the heavy ball algorithm and Nesterov's accelerated gradient algorithm can find an $\epsilon$-accuracy solution in $O\left(\sqrt{\frac{L}{\mu}}\log\frac{1}{\epsilon}\right)$ iterations, while the gradient descent algorithm needs $O\left(\frac{L}{\mu}\log\frac{1}{\epsilon}\right)$ iterations. Iteration (\ref{AGD1}) has an equivalent form of
\begin{eqnarray}\label{AGD2}
\y_{k+1}=\y_k-\sum_{j=0}^k h_{k+1,j}\nabla f(\y_j),
\end{eqnarray}
where
\begin{eqnarray}\label{h_AGD2}
h_{k+1,j}=\left\{
  \begin{array}{ll}
    \frac{\theta_{k+1}(1-\theta_k)}{\theta_k}h_{k,j}, & j=0,\cdots,k-2,\\
    \frac{\theta_{k+1}(1-\theta_k)}{\theta_k}(h_{k,k-1}-1), & j=k-1,\\
    1+\frac{\theta_{k+1}(1-\theta_k)}{\theta_k}, & j=k.
  \end{array}
\right.
\end{eqnarray}
ADMM and its linearized version can also be used to minimize $f(\z)$ by reformulating it as $\min_{\y,\z} f(\z)+f(\y),\quad s.t.\quad\y-\z=0$. Linearized ADMM consists the following steps\footnote{For direct use in our network design, we fix the penalty parameter to 1.}:
\begin{eqnarray}
\begin{aligned}\label{ADMM}
&\z_{k+1}=\argmin_{\z} \<\nabla f(\z_k),\z\>+\frac{1}{2}\|\z-\z_k\|^2+\<\lambda_k,\z\>+\frac{1}{2}\|\z-\y_k\|^2,\\
&\y_{k+1}=\argmin_{\y} \<\nabla f(\y_k),\y\>+\frac{1}{2}\|\y-\y_k\|^2-\<\lambda_k,\y\>+\frac{1}{2}\|\z_{k+1}-\y\|^2,\\
&\lambda_{k+1}=\lambda_k+(\z_{k+1}-\y_{k+1}).
\end{aligned}
\end{eqnarray}
\section{Modeling the Propagation in Feedforward Neural Network}\label{section4}
In the standard feedforward neural network, the propagation from the first layer to the last layer can be expressed as:
\begin{eqnarray}
\x_{k+1}=\Phi(\W_k\x_k),\label{information_flow}
\end{eqnarray}
where $\x_k$ is the output of the $k$-th layer, $\Phi$ is the activation function such as the sigmoid and ReLU, and $\W_k$ is a linear transformation. As claimed in Section \ref{section_constritution}, we do not consider the optimal weights during the structure design stage. Thus, we fix the matrix $\W_k$ as $\W$ to simplify the analysis.

We want to relate (\ref{information_flow}) with the gradient descent procedure (\ref{GD}). The critical step is to find an objective $F(\x)$ to minimize.
\begin{lemma}\label{F_lemma}
Suppose $\W$ is a symmetric and positive definite matrix\footnote{This assumption is just for building the connection between network design and optimization algorithms. $\W$ will be learnt from data once the structure of network is fixed.}. Let $\U=\sqrt{\W}$. Then there exists a function $f(\x)$ such that (\ref{information_flow}) is equivalent to minimizing $F(\x)=f(\U\x)$ using the following steps:
\begin{enumerate}
\item Define a new variable $\z=\U\x$,
\item Using (\ref{GD}) to minimize $f(\z)$,
\item Recovering $\x_0,\x_1,\cdots,\x_k$ from $\z_0,\z_1,\cdots,\z_k$ via $\x=\U^{-1}\z$.
\end{enumerate}
\end{lemma}
\begin{proof}
We can find $\Psi(z)$ such that $\Psi'(z)=\Phi(z)$ for the commonly used activation function $\Phi(z)$. Then we can have that $\nabla_{\z} \sum_i \Psi(\U_{i}^T\z)=\U\Phi(\U^T\z)=\U\Phi(\U\z)$. So if we let
\begin{eqnarray}
f(\z)=\frac{\|\z\|^2}{2}-\sum_i \Psi(\U_{i}^T\z),\label{objective_function}
\end{eqnarray}
where $\U_i$ is the $i$-th column of $\U$, then we have
\begin{eqnarray}
\nabla f(\z_k)=\z_k-\U\Phi(\U\z_k).\label{eq2}
\end{eqnarray}
Using (\ref{GD}) to minimize (\ref{objective_function}), we have
\begin{eqnarray}
\z_{k+1}=\z_k-\nabla f(\z_k)=\U\Phi(\U\z_k).\label{eq3}
\end{eqnarray}
Now we define a new function
\begin{eqnarray}
F(\x)=f(\U\x).\label{objective_function2}\notag
\end{eqnarray}
Let $\z=\U\x$ for variable substitution and minimize $f(\z)$ to obtain a sequence of $\z_0,\z_1,\cdots,\z_k$. Then we use this sequence to recover $\x$ by $\x=\U^{-1}\z$, which leads to
\begin{eqnarray}
\x_{k+1}=\U^{-1}\z_{k+1}=\U^{-1}\U\Phi(\U\z_k)=\Phi(\U\z_k)=\Phi(\U^2\x_k)=\Phi(\W\x_k).\notag
\end{eqnarray}
\end{proof}
We list the objective function $f(\x)$ for the commonly used activation functions in Table \ref{integral_table}.

\begin{table*}\label{integral_table}
\begin{center}
\footnotesize
\begin{tabular}{c|c|c}
\hline\hline
& Activation function & Optimization objective $f(\x)$\\\hline
Sigmoid & $\frac{1}{1+e^{-x}}$ & $\frac{\|\x\|^2}{2}-\sum_i\left[\U_{i}^T\x+\log\left(\frac{1}{e^{\U_{i}^T\x}}+1\right)\right]$\\
tanh    & $\frac{1-e^{-2x}}{1+e^{-2x}}$ & $\frac{\|\x\|^2}{2}-\sum_i\left[\U_{i}^T\x+\log\left(\frac{1}{e^{2\U_{i}^T\x}}+1\right)\right]$\\
Softplus& $\log(e^x+1)$ & $\frac{\|\x\|^2}{2}-\sum_i\left[C-\mbox{polylog}(2,-e^{\U_{i}^T\x})\right]$\\
Softsign& $\frac{x}{1+|x|}$ & $\frac{\|\x\|^2}{2}-\sum_i\phi_i(\x)$, where $\phi_i(\x)=\left\{
  \begin{array}{ll}
    \U_{i}^T\x-\log(\U_{i}^T\x+1), & \mbox{ if }\U_{i}^T\x>0,\\
    -\U_{i}^T\x-\log(\U_{i}^T\x-1), & \mbox{ otherwise }\\
  \end{array}
\right.$ \\
ReLU    & $\left\{
  \begin{array}{ll}
    x, & \mbox{ if }x>0,\\
    0, & \mbox{ if }x\leq 0.\\
  \end{array}
\right.$ & $\frac{\|\x\|^2}{2}-\sum_i\phi_i(\x)$, where $\phi_i(\x)=\left\{
  \begin{array}{ll}
    \frac{(\U_{i}^T\x)^2}{2}, & \mbox{ if }\U_{i}^T\x>0,\\
    0, & \mbox{ otherwise }\\
  \end{array}
\right.$\\
Leaky ReLU & $\left\{
  \begin{array}{ll}
    x, & \mbox{ if }x>0,\\
    \alpha x, & \mbox{ if }x\leq 0.\\
  \end{array}
\right.$ & $\frac{\|\x\|^2}{2}-\sum_i\phi_i(\x)$, where $\phi_i(\x)=\left\{
  \begin{array}{ll}
    \frac{(\U_{i}^T\x)^2}{2}, & \mbox{ if }\U_{i}^T\x>0,\\
    \frac{\alpha (\U_{i}^T\x)^2}{2}, & \mbox{ otherwise }\\
  \end{array}
\right.$\\
ELU & $\left\{
  \begin{array}{ll}
    x, & \mbox{ if }x>0,\\
    a(e^{x}-1), & \mbox{ if }x\leq 0.\\
  \end{array}
\right.$ & $\frac{\|\x\|^2}{2}-\sum_i\phi_i(\x)$, where $\phi_i(\x)=\left\{
  \begin{array}{ll}
    \frac{(\U_{i}^T\x)^2}{2}, & \mbox{ if }\U_{i}^T\x>0,\\
    a(e^{\U_{i}^T\x}-\U_{i}^T\x), & \mbox{ otherwise }\\
  \end{array}
\right.$\\
Swish & $\frac{x}{1+e^{-x}}$ & $\frac{\|\x\|^2}{2}-\sum_i\left[\frac{(\U_{i}^T\x)^2}{2}+\U_{i}^T\x\log\left(\frac{1}{e^{\U_{i}^T\x}}+1\right)-\mbox{polylog}\left(2,-\frac{1}{e^{\U_{i}^T\x}}\right)\right]$\\
\hline\hline
\end{tabular}
\end{center}
\caption{The optimization objectives for the common activation functions.}
\end{table*}

\section{From GD to Other Optimization Algorithms}\label{section5}
As shown in Section \ref{section4}, the propagation in the general feedfroward neural network can be seen as using the gradient descent algorithm to minimize some function $F(\x)$. In this section, we consider to use other algorithms to minimize the same function $F(\x)$.

$\mathbf{The}$ $\mathbf{Heavy}$ $\mathbf{Ball}$ $\mathbf{Algorithm}$. We first consider iteration (\ref{HB}). Similar to the proof in Section \ref{section4}, we use the following three steps to minimize $F(\x)=f(\U\x)$:
\begin{enumerate}
\item Variable substitution $\z=\U\x$.
\item Using (\ref{HB}) to minimize $f(\z)$, which is defined in (\ref{objective_function}). Then (\ref{HB}) becomes
\begin{eqnarray}
\z_{k+1}=\z_k-\nabla f(\z_k)+\beta(\z_k-\z_{k-1})=\U\Phi(\U\z_k)+\beta(\z_k-\z_{k-1}),\notag
\end{eqnarray}
where we use (\ref{eq2}) in the second equation.
\item Recovering $\x$ from $\z$ via $\x=\U^{-1}\z$:
\begin{eqnarray}\label{HB_network}
\begin{aligned}
\x_{k+1}=&\U^{-1}\z_{k+1}=\Phi(\U\z_k)+\beta(\U^{-1}\z_k-\U^{-1}\z_{k-1})\\
=&\Phi(\U^2\x_k)+\beta(\x_k-\x_{k-1})=\Phi(\W\x_k)+\beta(\x_k-\x_{k-1}).
\end{aligned}
\end{eqnarray}
\end{enumerate}
$\mathbf{Nesterov's}$ $\mathbf{Accelerated}$ $\mathbf{Gradient}$ $\mathbf{Descent}$ $\mathbf{Algorithm}$. Now we consider to use iteration (\ref{AGD1}). Following the same three steps, we have:
\begin{eqnarray}\label{AGD_network}
\x_{k+1}=\Phi\left(\W\left(\x_k+\beta_k(\x_k-\x_{k-1})\right)\right),
\end{eqnarray}
where $\beta_k=\frac{\theta_k(1-\theta_{k-1})}{\theta_{k-1}}$. Then we use iteration (\ref{AGD2}) to minimize $F(\x)$. Following the same three steps, we have
\begin{eqnarray}\label{AGD2_network}
\x_{k+1}=\sum_{j=0}^k h_{k+1,j}\Phi(\W\x_j)+\x_k-\sum_{j=0}^k h_{k+1,j}\x_j.
\end{eqnarray}

$\mathbf{ADMM}$. At last, we use iteration (\ref{ADMM}) to minimize $F(\x)$, which consists of the following steps:
\begin{eqnarray}
\begin{aligned}\label{ADMM_network}
&\x_{k+1}'=\frac{1}{2}\left(\Phi(\W\x_k')+\x_k-\sum_{t=1}^k(\x_t'-\x_t)\right),\\
&\x_{k+1}=\frac{1}{2}\left(\Phi(\W\x_k)+\x_{k+1}'+\sum_{t=1}^k(\x_t'-\x_t)\right).
\end{aligned}
\end{eqnarray}
Comparing (\ref{HB_network}), (\ref{AGD_network}) (\ref{AGD2_network}) and (\ref{ADMM_network}) with (\ref{information_flow}), we can see that the new algorithms keep the basic $\Phi(\W\x)$ operation but use some additional side paths, which inspires the new network structure design in Section \ref{section7}.
\section{Hypothesis: Faster Optimization Algorithm May Imply Better Network}\label{section6}
In this section, we consider the general representation learning task: Given a set of data points $\{\{\x_0^i,l_i\}:i=1,\cdots,m\}$, where $\x_0^i$ is the $i$-th data and $l_i$ is its label, we want to find a deep neural network which can learn the best feature $\f_i$ for each $\x_0^i$, which exists in theory but actually unknown in reality, such that $\f_i$ can perfectly predict $l_i$. For simplicity, we assume that $\x_0^i$ and $\f_i$ have the same dimension. As clarified at the beginning of Section \ref{section_constritution}, in this paper, we only consider the learning model form $\x_0^i$ to $\f_i$ and do not consider the prediction model from $\f_i$ to $l_i$. We do not consider the optimal weights during the structure design stage and thus, we study a simplified neural network model with the same linear transformation $\W\x$ in different layers. Actually, this corresponds to the recurrent neural networks \cite{putzky}. In this section we use $\{\x_0,\f\}$ instead of $\{\x_0^i,\f_i\}$ for simplicity.

\subsection{Same Linear Transformation in Different Layers}\label{section5.1}
We first consider the simplified neural network model with the same linear transformation $\W\x$ in different layers. As stated in Section \ref{section4}, the propagation in the standard feedforward network can be seen as using the gradient descent algorithm to minimize some function $F(\x)$. Assume that there exists a special function $F(\x)$ with some parameter $\U$ dependent on $\{\x_0,\f\}$ such that $\f=\argmin_{\x}F(\x)$. Now we use different algorithms to minimize this $F(\x)$ and we want to find the minimizer of $F(\x)$ via as few iterations as possible.

When we use the gradient descent algorithm to minimize this $F(\x)$ with initializer $\x_0$, the iterative procedure is equivalent to (\ref{information_flow}), which corresponds to the propagation in the feedforward neural network characterized by the parameter $\U$ discussed above. Let $\hat\f$ be the output of this feedforward neural network. As is known, the gradient descent algorithm needs $O\left(\frac{L}{\mu}\log\frac{1}{\epsilon}\right)$ iterations to reach an $\epsilon$-accuracy solution, i.e., $\|\hat\f-\f\|\leq\epsilon$. In other words, this feedforward neural network needs $O\left(\frac{L}{\mu}\log\frac{1}{\epsilon}\right)$ layers for an $\epsilon$-accuracy prediction.

When we use some faster algorithm to minimize this $F(\x)$, e.g., the heavy ball algorithm and Nesterov's accelerated gradient algorithm, their iterative procedures are equivalent to (\ref{HB_network}) and (\ref{AGD2_network}), respectively and the new algorithms will need $O\left(\sqrt{\frac{L}{\mu}}\log\frac{1}{\epsilon}\right)$ iterations for $\|\hat\f-\f\|\leq\epsilon$. That is, the networks corresponding to faster algorithms (also characterized by the same $\U$ but has different structures) will need fewer layers than the feedforward neural network discussed above.

We define a network with fewer layers to reach the same approximation accuracy as a better network. As is known, training a deep neural network model is a nonconvex optimization problem and it is NP-hard to reach its global minima. The training precess becomes more difficult when the network becomes deeper. So if we can find a network with fewer layers and no loss of approximation accuracy, it will make the training process much easier.
\subsection{Different Linear Transformation in Different Layers}\label{section5.2}
In the previous discussion, we require the linear transformation in different layers to be the same. This is not true except in recurrent neural networks and is only for theoretical explanation. Now we allow each layer to have a different transformation.

Assume that we have a network that is inspired by using some optimization algorithm to minimize $F(\x)$, where its $k$-th layer has the operation of (\ref{information_flow}), (\ref{HB_network}), (\ref{AGD_network}), (\ref{AGD2_network}) or (\ref{ADMM_network}). Denote its final output as $\mbox{Net}_{\U}(\x_0)$. Then for a network with finite layers, we have $\|\f-\mbox{Net}_{\U}(\x_0)\|\leq\epsilon$.

Now we relax the parameter $\U$ to be different in different layers. Then the output can be rewritten as $\mbox{Net}_{\U_1,\cdots,\U_n}(\x)$. Now we can use the following model to learn the parameters $\U_1,\cdots,\U_n$ with a fixed network structure:
\begin{eqnarray}
\min_{\U_1,\cdots,\U_n} \|\f-\mbox{Net}_{\U_1,\cdots,\U_n}(\x_0)\|^2\label{training_model}
\end{eqnarray}
and denote $\U_1^*,\cdots,\U_n^*$ as the solution. Then we have $\|\f-\mbox{Net}_{\U_1^*,\cdots,\U_n^*}(\x_0)\|\leq \|\f-\mbox{Net}_{\U}(\x_0)\|$, which means that different linear transformation in different layers will not make the network worse. In fact, model (\ref{training_model}) is the general training model for a neural network with a fixed structure.

\subsection{Simulation Experiment}
\begin{table}
\begin{center}
\begin{tabular}{c|c|c|c|c|c}
\hline
depth & ADMM (\ref{ADMM_network}) & GD (\ref{information_flow}) & HB (\ref{HB_network}) & AGD (\ref{AGD_network}) & AGD2 (\ref{AGD2_network}) \\
\hline
10 & 1.07576& 1.00644 & 1.00443 & 1.00270 & 1.00745\\
\hline
20 & 1.07495& 1.00679 & 1.00449 & 1.00215 & 1.00227\\
\hline
30 & 1.07665& 1.00652 & 1.00455 & 1.00204 & 1.00086\\
\hline
40 & 1.07749& 1.00653 & 1.00457 & 1.00213 & 0.99964\\
\hline
\end{tabular}
\caption{MSE comparisons of different optimization algorithm inspired neural network structures.}\label{table_sm}
\end{center}
\end{table}
In this section, we verify our hypothesis that the network structures inspired by faster algorithms may be better than the ones inspired by slower algorithms.

We compare five neural network structures, which are inspired by the gradient descent algorithm, the heavy ball algorithm, ADMM and two variants of Nesterov's accelerated gradient algorithm, which have the operation of (\ref{information_flow}), (\ref{HB_network}), (\ref{ADMM_network}), (\ref{AGD_network}) and (\ref{AGD2_network}) at each layer, respectively. We set $\beta=0.3$ for (\ref{HB_network}) and the parameters of (\ref{AGD_network}) and (\ref{AGD2_network}) are exactly the same with their corresponding optimization algorithms. We use the sigmoid function for $\Phi$ and $\W\x$ is a full-connection linear transformation. Then we use model (\ref{training_model})\footnote{The optimization algorithms we present in Section 2 are not related to what algorithm we use to train model (\ref{training_model}). The algorithms in Section 2 are for designing network structures.} to train the parameters of each layer under the fixed network structures. We generate 10,000 random pairs of $\{\x_0^i,\f_i\}$ in $N(\textbf{0},\textbf{I})$ as the training data. Each $\x_0^i$ and $\f_i$ has a dimension of 100. We use $\x_0^i$ as the input of the network and use its output to fit $\f_i$. We report the Mean Squared Error (MSE) loss value of the five aforementioned models with different depths after training 1,000 epoches.

Table \ref{table_sm} shows the experimental results. We can see that HB, AGD and AGD2 inspired neural network structures perform better than GD inspired network structure. This corresponds to the fact that the HB algorithm and AGD algorithm have a better theoretical convergence rate than the GD algorithm. The ADMM inspired network performs the worst. In fact, although ADMM has been widely used in practice, it does not have a faster theoretical convergence rate than GD. We also observe that the MSEs of GD, HB and AGD inspired network will not always decrease when the depth increases. This means that the deeper networks with GD, HB and AGD inspired structures are harder to train. However, the AGD2 inspired networks can still be efficiently trained with a larger depth when GD, HB and AGD inspired structures fail. This AGD2 is better may be because it has better numerical stability, although it is theoretically equivalent to AGD if there is no numerical error. Such a phenomenon is yet to be further explored.

\section{Engineering Implementation}\label{section7}
In the above section, we hypothesize and verify that faster optimization algorithm inspired neural network structure may need fewer layers without accuracy loss. In this section, we consider the practical implementations in engineering. Specifically, we consider the network structures inspired by algorithm iterations (\ref{information_flow}), (\ref{HB_network}), (\ref{AGD_network}), (\ref{AGD2_network}) and (\ref{ADMM_network}).

We define the following three meta operations for practical implementation.

$\mathbf{Relax}$ $\Phi$ $\mathbf{and}$ $\mathbf{W}$. We use $\W\x$ as the linear transformation with full-connection in Section \ref{section5}, which is the product of a matrix $\W$ and a vector $\x$. We may relax it to the convolution operation, which is also a linear transformation. Moreover, different layers may have different weight matrix $\W$ and $\W$ may not be a square matrix, thus the dimensions of input and output may be different. $\Phi$ is the nonlinear transformation defined by the activation function. It can also be relaxed to pooling and batch normalization (BN). Moreover, $\Phi(\cdot)$ can be a composite of nonlinear activation, pooling, BN, convolution or full-connection linear transformation. Using the different combinations of these operations, the network structure (\ref{information_flow}) covers many famous CNNs, e.g., LeNet \cite{lecun1998} and VGG \cite{simonyan-2015}. The activation function can also be different for different layers, e.g., be learnable \cite{yan-2016}.

In the following discussions, we replace $\Phi(\W\x)$ with the operator $T(\x)$ for more flexibility.

$\mathbf{Adaptive}$ $\mathbf{Coefficients}$. Inspired by (\ref{HB_network}), (\ref{AGD_network}) and (\ref{AGD2_network}), we can design some practical neural network structures. However, the coefficients in these formulations may be impractical.So we keep the structure inspired by these formulations but allow the coefficient to have other values or even be learnable. Specifically, we rewrite (\ref{HB_network}) as
\begin{eqnarray}\label{HB_network2}
\begin{aligned}
\x_{k+1}=T(\x_k)+\beta_1\x_k+\beta_2\x_{k-1},
\end{aligned}
\end{eqnarray}
where $\beta_1$ and $\beta_2$ can be set as any constants, e.g., 0, which means that we drop the corresponding term. It can also be co-optimized with the training of network's weights.

The structure of (\ref{HB_network2}) is illustrated in Figure \ref{structure}(a). The symbol $\bigoplus$ means that $\x_{k+1}$ is a combination of $T(\x_k)$, $\x_k$ and $\x_{k-1}$.

Now we consider (\ref{AGD_network}) and (\ref{AGD2_network}). Rewrite them as
\begin{eqnarray}\label{AGD_network2}
\begin{aligned}
\x_{k+1}=T(\beta_1\x_k+\beta_2\x_{k-1})
\end{aligned}
\end{eqnarray}
and
\begin{eqnarray}\label{AGD2_network2}
\begin{aligned}
\x_{k+1}=\sum_{j=0}^k \alpha_{k+1}^jT(\x_j)+\sum_{j=0}^k \beta_{k+1}^j\x_j.
\end{aligned}
\end{eqnarray}
All the coefficients $\alpha$ and $\beta$ can be set as any constants, e.g., 0 or following (\ref{AGD1}) or (\ref{h_AGD2}). They can also be co-trained with the weights of the network.

We demonstrate the structures of (\ref{AGD_network2}) and (\ref{AGD2_network2}) in Figure \ref{structure}(b) and \ref{structure}(c). From Figure \ref{structure}(b) we can see that it first makes a combination of $\x_k$ and $\x_{k-1}$ is then the operation $T$ follows. In Figure \ref{structure}(c), $\x_{k+1}$ combines all of $T(\x_1),\cdots,T(\x_k)$ and $\x_1,\cdots,\x_k$.

It is known that ResNet adds an additional skip connection that bypasses the non-linear transformations with an identity transformation:
\begin{eqnarray}
\x_{k+1}=T(\x_k)+\x_k.\notag
\end{eqnarray}
The structure of ResNet can be recovered from (\ref{HB_network2}) by setting $\beta_2=0$. If we also set $\beta_2=0$ in (\ref{AGD_network2}), then it is similar to the structure of ResNet. The difference is that (\ref{HB_network2}) performs the operation $T$ before adding the skip connection while (\ref{AGD_network2}) do it after the skip connection.

DenseNet is an extension of ResNet, which connects each layer with all its following layers. Consequently, the $k$-th layer receives the feature-maps of all preceding layers as its input, and produces
\begin{eqnarray}
\z_{k+1}=T([\z_0,\z_1,\cdots,\z_k]),\label{densenet}
\end{eqnarray}
where $[$ $]$ refers to the concatenating operation. (\ref{AGD2_network2}) recovers the structure of DenseNet by setting $\beta_{k}^j=0,\forall k,j$.

We can also rewrite (\ref{ADMM_network}) as
\begin{eqnarray}
\begin{aligned}\label{ADMM_network2}
\x_{k+1}'=&T(\x_k')+\sum_{t=1}^k\alpha_t\x_t'+\sum_{t=1}^k\beta_t\x_t,\\
\x_{k+1}=&T(\x_k)+\sum_{t=1}^k\alpha_t\x_t'+\sum_{t=1}^k\beta_t\x_t.
\end{aligned}
\end{eqnarray}
Figure \ref{structure}(d) demonstrates the structure of (\ref{ADMM_network2}). From the figure we can see that two paths interact with each other.
\begin{figure}
\begin{center}
    \includegraphics[width=0.9\linewidth]{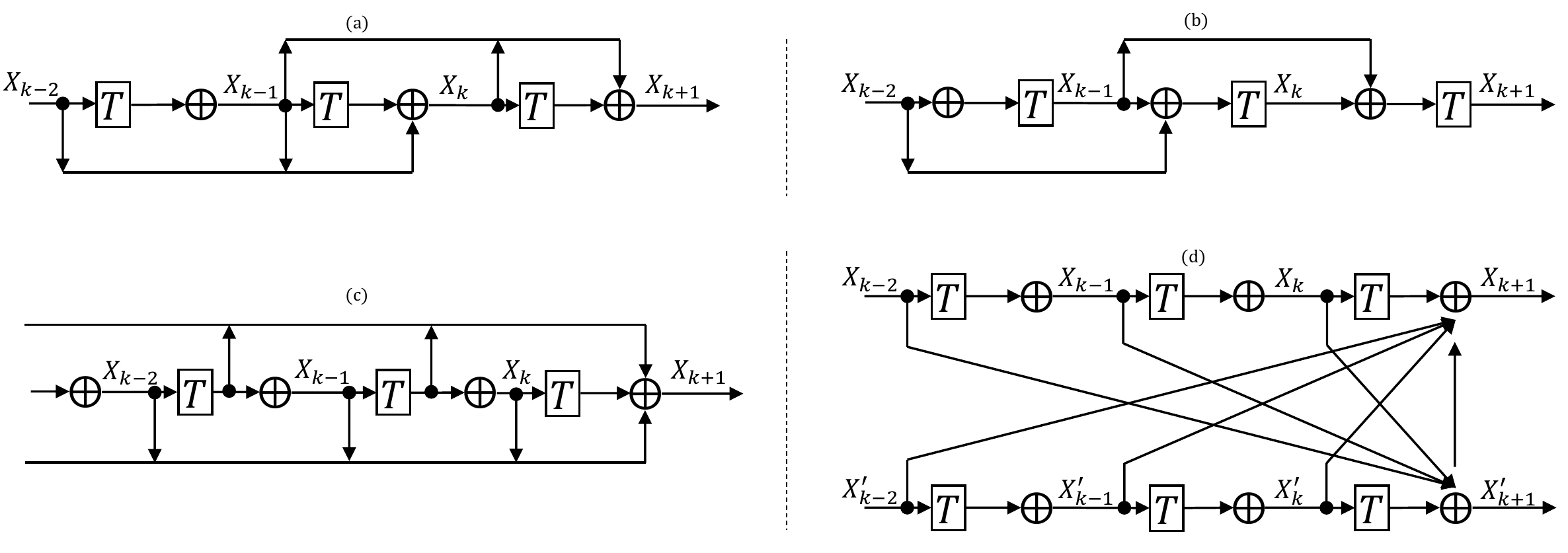}
\caption{Demonstrations of network structures (\ref{HB_network2}), (\ref{AGD_network2}), (\ref{AGD2_network2}) and (\ref{ADMM_network2}) in (a), (b), (c) and (d), respectively.}\label{structure}
\end{center}
\end{figure}
Different from the previous network structures, the ADMM inspired network has two parallel paths, which makes the network wider. It also recovers the DMRNet in \cite{wang2017}, which can be expressed as
\begin{eqnarray}
\begin{aligned}\notag
\x_{k+1}'=&T(\x_k')+\x_k'/2+\x_k/2,\notag\\
\x_{k+1}=&T(\x_k)+\x_k'/2+\x_k/2.\notag
\end{aligned}
\end{eqnarray}
We list the optimization algorithms that relate to the commonly used existing network structures in Table \ref{table_opt_stu}.

$\mathbf{Block}$ $\mathbf{Based}$ $\mathbf{Structure}$. (\ref{HB_network2}), (\ref{AGD_network2}) and (\ref{AGD2_network2}) are not available when the size of feature-maps changes, especially when down-sampling is used. To facilitate down-sampling, we can divide the network into multiple connected blocks. The formulationtions (\ref{HB_network2}), (\ref{AGD_network2}) and (\ref{AGD2_network2}) are used in each block. Such an operation is used in both ResNet and DenseNet.
\begin{table}
\begin{center}
\begin{tabular}{c|c|c}
\hline\hline
Algorithm & Network Structure & Transforming Setting\\
\hline
GD (\ref{GD}) & CNN & $\W\x\rightarrow$convolution\\
HB (\ref{HB}) & ResNet & $\beta_2=0$ in (\ref{HB_network2})\\
AGD (\ref{AGD2}) & DenseNet & $\beta=0$, $\alpha=1$ in (\ref{AGD2_network2})\\
ADMM (\ref{ADMM}) & DMRNet & $\alpha_k=\beta_k=\frac{1}{2}$ in (\ref{ADMM_network2})\\
\hline\hline
\end{tabular}
\end{center}
\caption{Optimization algorithms and their inspired network structures.}\label{table_opt_stu}
\end{table}
In the following sections, as examples we will give the explicit implementation of HB inspired network structure (\ref{HB_network2}) and AGD inspired network structure (\ref{AGD2_network2}).

\subsection{HB Inspired Network}\label{sec_HB_Resnet}
In this section, we describe the practical implementation of the neural network structure inspired by the heavy ball algorithm (\ref{HB}), 
which is used in our experiments. Specifically, we implement the HB inspired network by setting $\beta=1$ directly in (\ref{HB}):
\begin{equation}
\x_{k+1}=T(\x_k)+\x_k-\x_{k-1},\notag
\end{equation}
here $T$ is a composite function including two weight layers. According to the residual structure in ResNet, the first layer is composed of three consecutive operations: convolution, batch normalization, and ReLU, while the second one is performed only with convolution and batch normalization. The feature maps are down-sampled at the first layer of each block by convolution with a stride of $2$.
\subsection{AGD Inspired Network}
In line with the analysis above, we introduce the AGD inspired network of (\ref{AGD2_network2}) as follow, which is easy to implement.
\begin{equation}
\z_{k+1}=\sum^k_{j=0}\alpha_{k+1}^jT(\z_j) + \beta\left(\z_k-\sum^k_{j=0}h_{k+1}^j\z_j\right),
\end{equation}
where $T$ is the composite function including batch normalization, ReLU and convolution, following DenseNet. Different from DenseNet, where all preceding layers ($\z_j,j<k+1$) are concatenated first and then mapped by $T$ to produce $\z_{k+1}$, AGD inspired network makes each preceding layer $\z_j$ produce its own output first, and then sum the outputs by weights $\alpha_{k+1}^j$. The weights $\alpha_{k+1}^j$ of the first term, are co-optimized with the network, and the weights $h_{k+1}^j$ of the third term, are calculated by formulation (\ref{h_AGD2}). The parameter $\beta$ is set to be 0.1 in our experiments.

\section{Experiments}
\subsection{Datasets and Training Details}
\textbf{CIFAR} Both CIFAR-10 and CIFAR-100 datasets consist of $32\times32$ colored natural images. The CIFAR-10 dataset has 60,000 images in 10 classes, while the CIFAR-100 dataset has 100 classes, each of which containing 600 images. Both are split into 50,000 training images and 10,000 testing images. For image preprocessing, we normalize the images by subtracting the mean and dividing by the standard deviation. Following common practice, we adopt a standard scheme for data augmentation. The images are padded by 4 pixels on each side, filled with 0, resulting in $40\times40$ images, and then a $32\times32$ crop is randomly sampled from each image or its horizontal flip.

\textbf{ImageNet} We also test the vadility of our models on ImageNet, which contains 1.2 million training images, 50,000 validation images, and 100,000 test images with 1000 classes. We adopt standard data augmentation for the training sets. A $224\times224$ crop is randomly sampled from the images or horizontal flips. The images are normalized by mean values and standard deviations. We report the single-crop error rate on the validation set.

\textbf{Training Details} For fair comparison, we train our ResNet based models and DenseNet based models using training strategies adopted in the DenseNet paper \cite{huang2016densely}. Concretely, the models are trained by stochastic gradient descent (SGD) with 0.9 Nesterov momentum and $10^{-4}$ weight decay. We adopt the weight initialization method in \cite{he2015delving}, and use Xavier initialization \cite{glorot2010understanding} for the fully connected layer. For CIFAR, we train 300 epoches in total with a batchsize of 64. The learning rate is set to be 0.1 initially, and divided by 10 at 50\% and 75\% of the training procedure. For ImageNet, we train 100 eoches and drop learning rate at epoch 30, 60, and 90. The batchsize is 256 among 4 GPUs.
\begin{table}
\begin{center}
\begin{tabular}{l|c|c|c|c}
\hline
Model & CIFAR-10 & CIFAR-100 & CIFAR-10(+) & CIFAR-100(+) \\
\hline\hline
ResNet ($n=9$) & 10.05 & 39.65 & 5.32 & 26.03\\
HB-Net (\ref{HB_network2}) ($n=9$) & 10.17 & 38.52 & 5.46 & 26\\
\hline
ResNet ($n=18$) & 9.17 & 38.13 & 5.06 & 24.71\\
HB-Net (\ref{HB_network2}) ($n=18$) & 8.66 & 36.4 & 5.04 & 23.93\\
\hline
\hline
DenseNet ($k=12,L=40$)$^*$ & 7 & 27.55 & 5.24 & 24.42 \\

AGD-Net (\ref{AGD2_network2}) ($k=12,L=40$) & 6.44 & 26.33 & 5.2 & 24.87\\
\hline
DenseNet ($k=12,L=52$) & 6.05 & 26.3 & 5.09 & 24.33 \\

AGD-Net (\ref{AGD2_network2}) ($k=12,L=52$) & 5.75 & 24.92 & 4.94 & 23.84 \\
\hline
\end{tabular}
\end{center}
\caption{Error rates (\%) on CIFAR-10 and CIFAR-100 datasets and their augmented versions. `*' denotes the result reported by \cite{huang2016densely}. Others are implemented by ourselves. `+' denotes datasets with standard augmentation. We compare AGD-Net and HB-Net with DenseNet and ResNet, respectively, because they have very similar structures.}
\label{result_sec8}
\end{table}

\begin{table}
\begin{center}
\begin{tabular}{l|c|c}
          \hline\hline
Model & top-1(\%) & top-5(\%) \\
\hline
ResNet-34 & 26.73 & 8.74 \\
HB-Net-34 & 26.33 & 8.56 \\
\hline
DenseNet-121 & 25.02 & 7.71 \\
AGD-Net-121 & 24.62 & 7.39 \\
\hline\hline
\end{tabular}
\end{center}
\caption{Error rates ($\%$) on ImageNet when HB-Net and AGD-Net have the same depth as their baselines.}\label{xxxx}
\end{table}

\subsection{Comparison with State of The Arts}
The experimental results on CIFAR are shown in Table~\ref{result_sec8}, where the first two blocks are for ResNet based models and the last two blocks are for DenseNet based models. For ResNet based models, we conduct experiments with parameter $n$ of 9 and 18, corresponding to a depth of 56 and 110. For DenseNet based models, we consider two cases where the growth rate $k$ is 12 and the depth $L$ equals to 40 and 52, respectively. We do not consider a larger DenseNet model (\emph{e.g.} $L=100$) due to the memory constraint of our single GPU. We can see that our proposed AGD-Net and HB-Net have a better performance than their respective baseline. For DenseNet based models, when $k=12$ and $L=40$, our model has an improvement on all datasets except for augmented CIFAR-100. When $k=12$ and $L=50$, AGD-Net's superiority is obvious on all datasets. Similar to DenseNet based models, the superiority of HB-Net over ResNet increases as the model capacity goes larger from $n=9$ to $n=18$. Besides, as reported by the ResNet paper \cite{he2015resbet}, when $n=18$, the standard training strategy is difficult to converge and a warming-up is necessary for training. Our reimplementation of ResNet ($n=18$) in Table~\ref{result_sec8} indeed needs repeating the experiments to get a converged result. But for HB-net, we can use exactly the same training strategy adopted by other models and get a converged performance with training only once. Therefore, the training procedure of HB-Net is more stable than original ResNet when the network goes deeper.

As shown in Table~\ref{xxxx}, our proposed structures are also effective on the ImageNet dataset. Both HB-ResNet and AGD-Net has the better performance than their baselines.
All the above experiments show that our design methodology is very promising.
\section{Conclusion and Future Work}
In this paper, we use the inspiration from optimization algorithms to design neural network structures. We propose the hyphothesis that a faster algorithm may inspire us to design a better neural network. We prove that the propagation in the standard feedforward network with the same linear transformation in different layers is equivalent to minimizing some functions using the gradient descent algorithm. Based on this observation, we replace the gradient descent algorithm with the faster heavy ball algorithm and Nesterov's accelerated gradient algorithm to design better network structures, where ResNet and DenseNet are two special cases of our design.

Our methodology is still preliminary and not conclusive as many engineering tricks may also affect the performance of neural networks greatly. Nonetheless, our methodology can serve as the start point of network design. Practitioners can easily make changes to optimization algorithm inspired networks out of various insights and integrate various engineering tricks to produce even better results. Such a practice should be much easier than designing from scratch.


\small
\bibliographystyle{unsrt}
\bibliography{network}

\end{document}